\documentclass[letterpaper]{article} %
\usepackage[draft]{aaai2026}  %
\usepackage{times}  %
\usepackage{helvet}  %
\usepackage{courier}  %
\usepackage[hyphens]{url}  %
\usepackage{graphicx} %
\usepackage{natbib}  %
\usepackage{caption} %
\usepackage{algorithm}
\usepackage{placeins}

\usepackage{newfloat}
\usepackage{listings}
\usepackage{amsmath,amssymb}
\usepackage{booktabs}
\usepackage{subcaption}
\usepackage{multirow}
\usepackage{array}
\usepackage{macros}
\DeclareCaptionStyle{ruled}{labelfont=normalfont,labelsep=colon,strut=off} %
\floatstyle{ruled}
\newfloat{listing}{tb}{lst}{}
\floatname{listing}{Listing}
\defcitealias{socialScienceOneStatement2019}{Social Science One 2019}

\title{Manipulation-Proof Oblivious Audits against Deceptive Model Providers}
\author {
    Augustin Godinot\textsuperscript{\rm 1},
    Sofiane Azogagh\textsuperscript{\rm 2},
    Julien Ferry\textsuperscript{\rm 3},
    Sébastien Gambs\textsuperscript{\rm 4}
}
\affiliations {
    \textsuperscript{\rm 1}INRIA Centre de l'Université de Rennes, France\\
    \textsuperscript{\rm 2}EURECOM, Sophia Antipolis, France\\
    \textsuperscript{\rm 3}CIRRELT \& SCALE-AI Chair in Data-Driven Supply Chains, Department of Mathematics and Industrial Engineering, Polytechnique Montréal, Canada\\
    \textsuperscript{\rm 4}Université du Québec à Montréal, Canada\\
    augustin.godinot@inria.fr, sofiane.azogagh@eurecom.fr, julien.ferry@polymtl.ca, gambs.sebastien@uqam.ca
}

\newcommand{\revision}[1]{#1}

\begin{document}

\maketitle

\begin{abstract}
Audits have emerged as a critical instrument for algorithmic governance, providing a mechanism for external scrutiny and governance of machine learning models.
However, ensuring the integrity of such assessments remains a challenging issue.
For instance in regulatory contexts, audits are typically declared or easily detected, thus enabling model providers to manipulate the process, whether intentionally or inadvertently. 
This vulnerability is particularly acute in the context of fairness evaluations, in which providers can often infer sensitive attributes and strategically equalize allocation rates between groups to satisfy fairness metrics. 
In this paper, we introduce a novel audit protocol designed to significantly increase the post-audit detectability of such manipulations by enabling the auditor to query the model in an oblivious manner. 
Our approach leverages a Private Information Retrieval mechanism to require the provider to label a large set of instances, while preventing it from knowing which subset will ultimately be used for the audit. 
The protocol is efficient, imposes minimal overhead on the auditor, and requires no modification to the audited model, its training procedure, or its inference pipeline.
We provide theoretical guarantees showing that, under this protocol, a provider attempting to hide unfairness must falsify a significantly larger number of responses, thereby increasing both the difficulty and the likelihood of detection of manipulation. 
Experimental results across representative audit scenarios confirm the effectiveness and practicality of our approach.
\end{abstract}

\begin{links}
    \link{Code}{github.com/sofianeazogagh/oblivious_audit}
\end{links}

\section{Introduction}\label{sec:introduction}

\revision{
As machine learning (ML) systems increasingly mediate access to essential resources and public discourse, robust ML governance has become necessary to ensure these systems operate fairly and comply with legal and ethical standards~\cite{chandrasekaranSoKMachineLearning2021}. 
A central pillar of this governance is algorithmic auditing, which allows independent evaluators to scrutinize model behavior~\cite{metaxaAuditingAlgorithmsUnderstanding2021, costanza2022audits}. 
However, effective external auditing is currently hindered by a severe power asymmetry: model providers maintain strict control over system access, training data and usage statistics~\cite{birhaneAIAuditingBroken2024}. 
}

\revision{
While several mechanisms attempt to bridge this transparency gap—such as dedicated research APIs~\cite{entrena2025tiktok}, academic partnerships~\citepalias{socialScienceOneStatement2019} and legally mandated data access—they often face practical limitations and remain vulnerable to provider interference~\cite{bourreeRelevanceAPIsFacing2023}. 
A recent example highlights this fragility: the Meta Ad Library, a tool intended to help regulators and researchers track fraudulent advertisements via keyword searches. 
A 2025 Reuters' report revealed that Meta actively monitored the specific keywords and celebrity names used by Japanese regulators investigating scam ads~\cite{horwitzMetaCreatedPlaybook2025}. 
The company then repeatedly ran those exact searches to preemptively delete the targeted ads in Japan. 
By selectively scrubbing the specific data surfaced by the regulators' queries rather than deploying systematic scam filtering, the provider was able to avoid a regulatory crackdown while avoiding the cost of systematic ad verification. 
This reactive curation creates a false impression of compliance, demonstrating how easily current transparency mechanisms can be manipulated by providers when audits queries are observable by the auditee~\cite{fukuchiFakingFairnessStealthily2020,garciabourreeRobustMLAuditing2025}.
}

\revision{
In this work, we consider a generic regulation setting that captures the dynamics illustrated by the Ad Library example.
A regulator wishes to audit a company providing a prediction service, whether a prediction API or a model used as part of a larger system. 
We assume that the regulator can impose an audit protocol on the provider, provided a reasonable computational cost. 
The provider is modeled as a \emph{malicious} adversary: rather than incurring the cost of genuinely improving its system, it seeks to pass the audit by adversarially flipping a number of predictions, without modifying its deployed model.
Indeed, the central challenge in this setup (regulator-led audits, requiring provider cooperation) is that the audited provider may strategically modify its outputs when answering audit queries, deviating from the outputs it would provide to genuine users.
Not only is this type of manipulation difficult to detect \cite{garciabourreeRobustMLAuditing2025}, but a few modifications of the outputs on the audit set can have a dramatic impact on the audit conclusions (see \Cref{thm:n_manipulation_vanilla_audit}).}
\revision{To address this, we propose a rigorous auditing framework that offers formal theoretical guarantees on the audit’s conclusions and limits the extent of strategic manipulation the provider can perform. While our proposed framework is model- and metric-agnostic, we restrict our presentation to auditing demographic parity in binary classifiers for simplicity, and discuss in \Cref{sec:applicability} how it extends to other settings.
Our main contributions are as follows:
\begin{itemize}
    \item We introduce \acs{respir}, an audit protocol designed to mitigate audit manipulations. 
    The core idea is to hide the actual audit set $S$ within a larger candidate set $C$, provably making manipulation harder for a deceptive provider. The only prerequisite for this protocol is that the audited provider must label a candidate set $C$ larger than the actual audit set $S$.
    \item We prove that \acs{respir} forces a deceptive provider to modify its answers on a substantially larger number of points in the candidate set $C$ in order to fake the same level of demographic parity (\Cref{thm:n_manipulation_respir_audit}). This results in an increased probability of manipulation detection by the auditor.
    \item We characterize how large the candidate set $C$ must be, relative to the audit set $S$, to achieve a desired level of manipulation detection while demonstrating that the protocol does not incur prohibitive computational overhead, making it practical for real-world deployment.
\end{itemize}
}

\paragraph{Outline.}
\revision{First, in \Cref{sec:related_works} we review the background on audits, attempts at making them robust, and Private Information Retrieval, which serves as a core building block of our approach.} Next, in \Cref{sec:auditing_game} we describe the auditing framework on which our work relies, and we explain why standard black-box fairness audits can be vulnerable to manipulation. To address this issue, we introduce in \Cref{sec:o_audit} an oblivious auditing protocol based on PIR, which we call \acs{respir}. We then establish the protocol’s theoretical guarantees in terms of manipulation effort and detectability, and evaluate it experimentally in \Cref{sec:expes}. Finally, we conclude in \Cref{sec:conclusion} with a discussion \revision{regarding our framework's applicability,} the main design principles for hidden-audit-set protocols, and directions for future work.

\section{Background \& Related Work}\label{sec:related_works}

\subsection{\revision{Audits and Manipulations}}
\revision{
\emph{Audits and regulation.}
Independent audits are a key mechanism for holding model providers accountable and fostering healthier algorithmic ecosystems~\cite{metaxaAuditingAlgorithmsUnderstanding2021, vecchione2021algorithmic, costanza2022audits}. 
This view has been adopted by recent regulations: the EU AI Act~\cite{RegulationEU20242024}, DSA~\cite{RegulationEU20222022}, DMA~\cite{RegulationEU20222022a} and related frameworks increasingly mandate external assessments of high-risk AI systems, making algorithmic auditing not only a research concern but a legal requirement.
}

\emph{Practical challenges to auditing.} 
Several incidents highlight the practical difficulties of conducting trustworthy audits.
In some cases, the studied algorithm changed without notice during the study, as in the 2020 Facebook polarization study~\cite{ribeiroFacebookStandardAlgorithm2024}.
In others, promised (or required) data access was repeatedly delayed~\citepalias{socialScienceOneStatement2019} or simply not provided~\cite{EuropeanCommission2025}.
Finally, when data or models are eventually shared, they may contain errors that undermine published academic work~\cite{timberg_facebook_2021} or fail to reflect the reality experienced by users~\cite{entrena2025tiktok, pearson_beyond_2025}.

\revision{
\emph{Manipulability of audit targets.}
Beyond practical obstruction, a growing body of work demonstrates that many components of ML systems can be strategically %
manipulated. 
Training data and sampling procedures can be biased to fake fairness~\cite{fukuchiFakingFairnessStealthily2020}. 
Model explanations can be fairwashed to conceal discriminatory behavior~\cite{aivodjiFairwashingRiskRationalization2019, andersFairwashingExplanationsManifold2020, shamsabadiWashingUnwashableImpossibility2022,fokkemaAttributionbasedExplanationsThat2023}.
Model outputs themselves can be selectively altered to deceive black-box auditors~\cite{bourreeRelevanceAPIsFacing2023, garciabourreeRobustMLAuditing2025}. 
And more broadly, fairness metrics are susceptible to gaming when providers can anticipate the evaluation procedure~\cite{hutchinson2022evaluation, thomas2022reliance}.
}

\subsection{\revision{Related Work on Robust Auditing}}
\revision{Existing work on making audits more robust to manipulation can be organized into three broad categories.}

\revision{
\emph{(1) Data-driven and prior-based approaches.}
Several works propose auditing protocols that leverage prior knowledge or careful query design to limit manipulation. 
\citet{yanActiveFairnessAuditing2022} study active fairness auditing, where the auditor adaptively selects queries to efficiently estimate fairness violations. 
\citet{yadavXAuditTheoreticalLook2023} provide a theoretical framework for auditing with explanations. 
Most closely related to our work, \citet{garciabourreeRobustMLAuditing2025} formally study manipulation-proof auditing and show that auditors with prior knowledge about the ground truth can detect manipulations, provided the provider does not know the exact prior they use.
Complementary to these approaches, our method allows to relax the private prior assumption by cryptographically hiding the auditor's queries.
}

\revision{
\emph{(2) Cryptographic approaches.}
A second line of work uses cryptographic tools to provide verifiable guarantees. 
\citet{shamsabadiConfidentialPROFITTConfidentialPROof2023} propose zero-knowledge proofs of fair training for decision trees. 
\citet{waiwitlikhitTrustlessAuditsRevealing2024} design a protocol based on zero-knowledge proofs that enables trustless audits without revealing models or data. 
\citet{franzese_secure_2025} introduce online fairness certificates via a scalable zero-knowledge proof protocol that verifies fairness with respect to data received during deployment. 
\citet{pentyalaPrivFairLibraryPrivacyPreserving2022} develop privacy-preserving fairness auditing using secure computation. 
These approaches offer strong guarantees but typically require the provider's active cooperation in running the cryptographic protocol and may impose significant computational overhead. In contrast, our proposed framework relies on lightweight PIR primitives and incurs only negligible overhead.
}

\revision{
\emph{(3) Secure hardware.}
A third category relies on trusted execution environments or secure hardware to ensure audit integrity~\cite{park2022fairness}. While promising, these approaches require trust in the hardware manufacturer and physical deployment constraints, which limit their applicability in many regulatory contexts.
}

\subsection{Private Information Retrieval}
\label{sec:PIR}

Private Information Retrieval (PIR) is a cryptographic primitive that enables a client to retrieve an item from a database without revealing to the server which item is being accessed~\cite{ChorKGS98}. 
A large body of work has proposed different PIR constructions depending on the trust and deployment assumptions. 
One approach distributes the database across multiple non-colluding servers to obtain information-theoretic privacy guarantees~\cite{OlumofinG11}. 
Other approaches rely on trusted hardware~\cite{SmithS01} but the most promising ones are homomorphic encryption schemes based on lattice problems such as Learning With Errors (LWE)~\cite{xpir16,sealpir18,mulpir21,fastpir21,onionpir21,spiral22,frodopir23,simplepir23}. 
In these computational PIR based on homomorphic encryption, the client typically encodes its query as an encrypted selection vector, conceptually corresponding to a one-hot vector indicating the desired record index.
Afterwards, the server homomorphically combines this encrypted query with the database contents, which are generally represented as a matrix, which means that answering a query essentially reduces to an encrypted matrix-vector multiplication.
The resulting ciphertext decrypts to the requested record while the server learns nothing about the queried index.
However, this approach generally incurs a computational cost linear in the database size for each query, which severely limits its practicality for large-scale databases.
To overcome this limitation, a recent line of work has introduced so-called \emph{stateful PIR} schemes, in which a significant portion of the computation and communication is shifted to an offline preprocessing phase~\cite{corrigan2020private,kogan2021private,shi2021puncturable,onionpir21,corrigan22single}.
In this stateful PIR, the server first computes and sends to the client a public or semi-public hint (or \emph{digest}) that depends solely on the database and can be reused across multiple queries.
Once this preprocessing step is completed, each individual query can be answered with substantially reduced communication complexity and the server becomes sublinear in the database size.

Beyond achieving sublinear asymptotic costs, stateful PIR techniques have also been leveraged, not necessarily to achieve sublinear costs, but to make the linear cost per query significantly more efficient~\cite{simplepir23,frodopir23}. 
A prominent example is SimplePIR~\cite{simplepir23}, which relies on Regev's LWE-based encryption scheme~\cite{regev2009lattices}.
The key insight of SimplePIR is to exploit the linear structure of Regev encryption to precompute the vast majority of the expensive matrix operations associated with the database during the offline phase.
As a result, the online query phase only requires lightweight computations that depend on the encrypted query, while the bulk of the matrix-vector multiplication is carried out in advance and depends exclusively on the database.
This property makes SimplePIR particularly well suited to our setting, in which the provider is informed in advance of an upcoming audit and jointly agrees with the auditor on the audit set.
Once this agreement is established, the provider can label the audit set accordingly and transmit a corresponding digest to the auditor (see details in Section~\ref{sec:o_audit}).
Furthermore, De Castro and Lee~\cite{verisimplepir24} introduced VeriSimplePIR, an extension of SimplePIR that ensures verifiability.
More precisely, it allows the client (\emph{i.e.}, the auditor in our case) to verify that the database queried is exactly the one used to generate the digest sent by the server (\emph{i.e.}, the provider in our setting) without disclosing the content of the database. 
This verifiability is achieved at no additional online cost and can be used to prevent the provider from changing the labels during the audit.

\section{The Black-Box Auditing Game}
\label{sec:auditing_game}

In this section, we first describe the typical black-box fairness auditing setup considered in this paper. 
Then, we theoretically quantify its vulnerability to manipulation.

\subsection{Black-Box Fairness Auditing}%
\label{subsec:black-box_audit}

Let $\mathcal{X}$ be the input space, and let $h:\mathcal{X}\to\{0,1\}$ denote the provider's (binary) prediction function.
Each query $x\in\mathcal{X}$ is associated with a binary \emph{protected} attribute $A(x)\in\{0,1\}$, which partitions the population into two \emph{protected groups} (those with $A=1$ and those with $A=0$).

The \emph{demographic parity} gap of $h$ quantifies how differently $h$ assigns positive outcomes across these two groups.
More precisely, it is defined as the difference between the probability of predicting the positive class for individuals in group $A=1$ and the corresponding probability for individuals in group $A=0$:
\[
\begin{aligned}
\metric_{\mathcal{D}}(h)
&:= \P[X\sim\mathcal{D}]{h(X)=1 \given A(X)=1} \\
&\quad - \P[X\sim\mathcal{D}]{h(X)=1 \given A(X)=0},
\end{aligned}
\]
in which $\mathcal{D}$ denotes the underlying data distribution. 
In practice, since $\mathcal{D}$ is typically unknown, this quantity is estimated on a finite sample set $S\subset\mathcal{X}$. 
For $a\in\{0,1\}$, let $n_a = \bigl|\{x\in S: A(x)=a\}\bigr|$ 
and define the empirical positive prediction rate over protected group $a$ as
\[
p_{a,S}(h)
\;:=\;
\frac{1}{n_a}\sum_{x\in S_a} h(x).
\]
The demographic parity gap of $h$ on $S$ is then
\[
d_S(h) \;:=\; p_{1,S}(h) - p_{0,S}(h).
\]

When the classifier \(h\) is clear from the context, we omit it from the notation and simply write \(p_{a,S}\) and \(d_S\).
Given a tolerance \(\varepsilon \ge 0\), we say that \(h\) satisfies the demographic parity constraint on \(S\) if
\begin{align}
    |d_S|\le \varepsilon, \label{eq:constr_fairness}
\end{align}
in which \(\varepsilon\) is a chosen tolerance threshold.
The objective of the auditor is to verify whether Constraint~\eqref{eq:constr_fairness} holds for a given \emph{audit set} \(S\).
More precisely, if \(|d_S|\le \varepsilon\), the ML provider passes the test while otherwise, if \(|d_S| > \varepsilon\), the provider is flagged as unfair by the auditor.
However, a key challenge is that the auditor often does not have access to the model \(h\) itself, nor to its training data, code or documentation.
To estimate the demographic parity gap on an audit set \(S\), the auditor must therefore \emph{resort to black-box access} and query \(h\) on a set of inputs \(S \subset \mathcal{X}\) that they collect. 
The resulting predictions are then used to compute \(d_S\). 
This black-box audit procedure is described in \Cref{algo:bbox_audit}.
In the next subsection, we show how this framework is vulnerable to provider manipulation.

\subsection{Manipulating Black-box Audit}

One important limitation of typical demographic parity audits conducted as described in Section~\ref{subsec:black-box_audit} is that the auditor only has access to query answers (through the provider's audit API) and no knowledge of the underlying model.
It is therefore possible for a malicious provider to adversarially manipulate the audit by returning a label \(h(x)\) that differs from the underlying model’s true output.
To quantify how easy such manipulation is, Proposition~\ref{thm:n_manipulation_vanilla_audit} characterizes the minimum number of predictions that the model provider must flip in order to fall below the detection threshold.

\begin{proposition}[Black-box audit manipulation]\label{thm:n_manipulation_vanilla_audit}
Let $S \subset \mathcal{X}$ be the audit set (with $\lvert S \rvert = n$), let \(d_{\text{true}}\) denote the demographic parity gap of the provider's model on \(S\) (in the absence of manipulation), let \(n_{\min} = \min(n_0,n_1)\) be the size of the smallest protected group in \(S\), and let \(\mtol\) be the audit threshold.
Assuming the model would fail the audit, \emph{i.e.}, \(|d_{\text{true}}| > \mtol\), the minimum number \(m_{\text{vanilla}}\) of predictions that the provider must flip in order to pass the audit satisfies
\begin{equation}
    m_{\text{vanilla}} \;\ge\; \lceil \bigl(|d_{\text{true}}| - \mtol\bigr)\, n_{\min} \rceil.
\end{equation}
\begin{proof}
Assume without loss of generality that $d_{\text{true}}>0$. 
Flipping a prediction $h(x)$ from $1$ to $0$ for an example with $A(x)=1$ decreases $p_{1,S}$, and hence the gap on $S$, by at most $1/n_1$. 
Similarly, flipping a prediction from $0$ to $1$ for an example with $A(x)=0$ increases $p_{0,S}$ and decreases the gap by at most $1/n_0$. 
Since a single flip affects only one group rate, the gap can decrease by at most
\[
\max\!\left(\frac{1}{n_0},\frac{1}{n_1}\right)
=
\frac{1}{n_{\min}}
\]
per modification. 
Therefore, reducing the gap from $d_{\text{true}}$ to at most $\mtol$ requires at least $(d_{\text{true}}-\mtol)\,n_{\min}$ modifications, hence at least $\lceil(d_{\text{true}}-\mtol)\,n_{\min}\rceil$ since $m_{\text{vanilla}}$ is an integer.
\end{proof}
\end{proposition}

\begin{figure*}[t!]
    \centering
    \begin{minipage}[t]{0.39\linewidth}
        \vspace{0pt} %
        \begin{algorithm}[H]
            \caption{Black-Box audit procedure\vphantom{\acs{respir}}}
            \label{algo:bbox_audit}
            \footnotesize
            \begin{algorithmic}[1]
                \Require $\inspace$, $\model$
                \State Auditor gathers or generate an audit dataset $S \subset \inspace$ \label{algo:bbox_audit:audit_set}
                \State For each example $x \in S$, the auditor queries the provider for $h(x)$
                \State The auditor computes the parity gap $d_S$
                \If{$d_S < \mtol$}
                    \State \Return \texttt{Pass} 
                \Else~
                    \State \Return \texttt{Fail}
                \EndIf
            \end{algorithmic}
        \end{algorithm}
    \end{minipage}
    \hfill%
    \begin{minipage}[t]{0.59\linewidth}
        \vspace{0pt} %
        \begin{algorithm}[H]
            \caption{The \acs{respir} audit procedure}
            \label{algo:respir}
            \footnotesize
            \begin{algorithmic}[1]
                \Require $\inspace$, $\model$
                \State Provider and auditor agree on a large candidate set $C \subset \inspace$ \label{algo:respir:candidates}
                \State Provider privately labels all examples in $C$, creating $D = \set{(x, h(x)) : x \in C}$ \label{algo:respir:labeling}
                \State Auditor interacts with $D$ to retrieve outputs on an audit set $S \subset D$ \label{algo:respir:PIR_interaction}
                \Statex using the \ac{PIR} primitive
                \If{$d_S < \mtol$} \label{algo:respir:decision_1}
                    \State \Return \texttt{Pass} \label{algo:respir:decision_2}
                \Else~
                    \State \Return \texttt{Fail}\label{algo:respir:decision_3}
                \EndIf
            \end{algorithmic}
        \end{algorithm}
    \end{minipage}%
\end{figure*}

Proposition~\ref{thm:n_manipulation_vanilla_audit} provides a lower bound on the number %
of output modifications that a provider must perform in order to pass the audit.
This bound is linear in the amount of unfairness to be hidden but also in the size of the smallest protected group in the audit set, which means that unbalanced audit sets therefore greatly facilitate manipulation.
However, even for a balanced audit set, manipulating the audited fairness remains easy.
For instance, to hide a disparity of \(0.01\) (\emph{i.e.}, \(d_{\text{true}}-\mtol=0.01\)) with a balanced audit set of \(n=400\) queries, the provider only has to modify two predictions.

The manipulation model considered in this work therefore consists of arbitrary output flipping. By allowing any (worst-case) modification to the prediction vector returned by the platform, this model generalizes more constrained threat models and yields model-agnostic guarantees.

Assume that the auditor possesses a small \emph{verification set} of \(k\) queries (also called \emph{canaries}).
These queries are externally labeled (\emph{e.g.}, via crowdsourcing or by directly querying the provider's public API), and manipulation is detected as soon as at least one verification query is modified by the provider.
Since the auditor controls the audit set \(S\), we assume that the canaries are included in \(S\) by construction.
Therefore, the canary effectiveness can be modeled through a single \emph{catch-all} parameter \(q \in [0,1]\), interpreted as the probability that a canary included in the audit set remains indistinguishable to the provider (so that the provider cannot selectively avoid modifying it).
Let \(n \;:=\; \lvert S \rvert\) denote the size of the audit set, and let \(m\) denote the number of audit set queries whose outputs are adversarially modified by the provider.
Then, conditional on being effectively present, a canary is modified with probability \(m/n\), and the probability that manipulation is detected satisfies
\begin{equation}
\mathbb{P}(\text{manipulation detected})
\;=\;
1 - \bigl(1 - q\, m/n\bigr)^k. \label{eq:manipulation_detection_proba}
\end{equation}
Overall, this highlights that typical black-box fairness audits can be highly brittle, as a small number of targeted output modifications may suffice to hide unfairness.
Moreover, our detection analysis shows that the probability of catching such manipulation depends critically on both the number of adversarial manipulations $m$ and the size \(k\) of the verification set: small \(k\) provides limited guarantees when only a few outputs are modified, whereas detection quickly becomes likely as either \(m\) or \(k\) increases. 

While increasing the number of canaries \(k\) is a simple countermeasure, it is often impractical or too costly in practice.
Instead, in the next section, we propose an oblivious auditing protocol that prevents the provider from tailoring its modifications to the specific audit set.
As a result, to hide unfairness with high probability, the provider is forced to perform a larger number \(m\) of output modifications, which increases the cost of manipulation and its detectability.

\section{Oblivious Auditing}\label{sec:oblivious_auditing}
\label{sec:o_audit}

\subsection{\acs{respir} Protocol}
Based on the observation that hiding unfairness over a known finite set is easily achieved, we propose a new audit protocol coined \ac{respir}.
To mitigate audit manipulations, \ac{respir} forces the provider to find a strategy to optimize its manipulations on a significantly larger \emph{candidate set} $C$ rather than on the particular audit set $S$ drawn by the auditor. 
The cryptographic primitive that prevents the provider from knowing which sample is used in the audit is the \ac{PIR} (\emph{cf.} Section~\ref{sec:PIR}).
The audit protocol consists of three steps, described in \Cref{algo:respir}:
\begin{itemize}
\item{Step 1 (\Cref{algo:respir}, \cref{algo:respir:candidates}).}
The provider and the auditor agree on a (large) \emph{candidate set} $\C$. 
In practice the candidate set can be proposed by the auditor (\emph{e.g.}, derived from a large public dataset) or by the provider (\emph{e.g.}, the company is required to share internal information to prepare for the audit).
The selection process for $\C$ can even be iterative: after one party proposes a first candidate set, the other can ask for additions or deletions to reach a more representative distribution.

\item{Step 2 (\Cref{algo:respir}, \cref{algo:respir:labeling,algo:respir:PIR_interaction})}
The auditor samples an \emph{audit set} $\S \subset \C$ among the candidate set and runs it through the PIR mechanism.
The \ac{PIR} mechanism guarantees that 1) the provider cannot infer the exact audit queries used in $S$ and 2) the auditor cannot use the interaction to gather more information about the output of $h$ other than what is given by $h(\S)$.
Thus, the \ac{PIR} properties effectively exclude any manipulation during the interaction between the auditor and the provider at step 2.

\item{Step 3 (\Cref{algo:respir}, \cref{algo:respir:decision_1,algo:respir:decision_2,algo:respir:decision_3})}
Using the answers of the provider $h(\S)$ gathered during the \ac{PIR} phase, the auditor produces an estimate $d_S$ of the audit metric and decides whether the provider passes the audit or not.
The goal of a malicious provider is to make the auditor accept its model (\emph{i.e.}, $d_S \leq \mtol$) when it should not have passed (\emph{i.e.}, $d_{\text{true}} > \mtol$).
\end{itemize}

We further highlight the following design properties.
First, the candidate set \( C \) is constructed to be representative of the (unknown) data distribution \( \mathcal{D} \).
Consequently, when the audit set \( S \subseteq C \) is sampled uniformly at random and is sufficiently large, it is expected to closely reflect \( \mathcal{D} \).
\revision{While $S$ could be sampled from $C$ in other ways, uniform sampling is used because it (i) yields representativeness via standard concentration arguments and is easier for the provider to accept, whereas tailored sampling can appear arbitrary to target specific fairness values, (ii) requires no prediction-dependent design since the auditor sees features in $C$ but not the model’s outputs, and (iii) enables a clean probabilistic derivation of manipulation lower bounds (Theorem~\ref{thm:n_manipulation_respir_audit}).}
Moreover, planting a small number \(k\) of verification queries (\emph{canaries}) into \(S\) does not materially affect this representativeness when \(k \ll |S|\).
As a result, it is reasonable to require the provider to pass the fairness verification (\emph{i.e.,} to exhibit unfairness of at most \( \varepsilon \)) on \( S \).

Second, although the provider is required to label a potentially large candidate set \( C \), the amount of information revealed to the auditor is the same as in the traditional black-box auditing setup described in Section~\ref{subsec:black-box_audit}, since the sampled audit set \( S \) need not be larger.
Because the provider’s underlying model may be proprietary or constitute a business secret, this property is important in practice.

Third, the use of a \ac{PIR} scheme prevents the provider from tailoring its responses to a known audit set.
Instead, to pass the audit with high probability, the provider must ensure that the fairness constraint holds for a random audit set drawn from the candidate set, implying that fairness must generalize across the candidate set rather than being satisfied only on a specific subset. 
In \Cref{sec:manip}, we provide a probabilistic analysis showing that this substantially increases both the cost and detectability of manipulation.

\subsection{VeriSimplePIR Applied to an Audit Set}
Let $S$ denote the audit set and $h$ the provider’s prediction function. 
The primary advantage of using a PIR protocol for auditing, as opposed to sending an encrypted version of the audit set $S$, lies in its efficiency as PIR is considerably faster than evaluating the model on $S$ via homomorphic encryption. 
For example, if $h$ is a neural network, running $x \in S$ through $h$ with homomorphic encryption would require as many matrix-vector multiplications as there are network layers. 
In contrast, the PIR protocol reduces this to a single, lightweight matrix-vector multiplication per $x \in S$. 

While a detailed description of the interaction during the PIR between the auditor and the provider is provided in \Cref{fig:verisimplepir} in Appendix~\ref{appendix:pir_scheme}, we summarize here the main steps of the protocol at a high level.
VeriSimplePIR proceeds in two phases: an offline phase for setup and commitment and an online phase for querying. 
The process starts when the auditor requests an audit from the provider. After both parties agree on the audit set $S$, the provider applies its prediction function $h$ to label the samples, resulting in a database $h(S)$. 
These labels are then packed into a small matrix $\D \in \Z_p^{\ell \times \mu}$, with the parameters $\ell$, $p$, and $\mu$ specified by the auditor to satisfy security requirements. 
Additionally, the provider provides an indexing table to assist the auditor in mapping labels to the corresponding samples in $\D$.

The protocol then proceeds through four main functions: \textbf{Commitment}, \textbf{Query}, \textbf{Answer} and \textbf{Recovery}. 
During the offline phase, the provider performs the \textbf{Commitment} function by generating a cryptographic digest of the matrix $\D$ and a ``hint'' sent to the auditor to verify that $\D$ is well formed. 
The auditor cannot learn the content of the database from the digest or the hint. 
Optionally, the provider can also precompute a proof $\pi$ that will later be used to verify that queries are performed on the committed database (see Construction 5.1 in~\cite{verisimplepir24} for more details about this proof).

In the online phase, the auditor executes the \textbf{Query} function by selecting an index $i=(i_r,i_c)$ and sending an encryption of the one-hot encoding of $i_c$ to the provider using the Regev's LWE encryption scheme~\cite{regev2009lattices}. 
The provider then performs the \textbf{Answer} function by computing a matrix-vector multiplication between the $\D$ and the encrypted query. 
Finally, the auditor executes the \textbf{Recovery} function to decrypt and extract the queried value from the response. 
The auditor can perform multiple queries and verify them all at once using the precomputed proof $\pi$ (see~\cite{verisimplepir24} for more details about this proof).
To demonstrate how the protocol performs in practice, \Cref{tab:performances} reports timing measurements for the online phase (query, answer and recovery steps) along the bandwidth usage, across various sizes of $h(S)$, in which $h$ is a binary classifier. The given results are an average of 10 runs.
When $\forall y \in h(S), \log_2(y) < \log_2(p)$, multiple labels $y$ can be encoded within a single element in $\Z_p$, so each entry in $\D$ can store several labels. 
As a result, the matrix $\D$ can be considerably smaller than the raw set of labels $h(S)$. 
The results reported shows that the PIR protocol maintains practical efficiency even when handling large audit sets.

\begin{table}[t!]
  \centering
  \caption{Performance of VeriSimplePIR~\cite{verisimplepir24} on the online phase for different database sizes.}
  \label{tab:performances}
  \resizebox{0.98\linewidth}{!}{
  \begin{tabular}{lccccc}
  \toprule
  & \multicolumn{5}{c}{Database size} \\
  \cmidrule(lr){2-6}
  & 128 KiB & 128 MiB & 1 GiB & 4 GiB & 8 GiB  \\
  \midrule
  Query size    & 0.375 KiB    & 13.16 KiB   & 37.76 KiB   & 77.85 KiB   & 110.1 KiB   \\
  Answer size   & 7.812 KiB    & 276.6 KiB   & 816.8 KiB   & 1.644 MiB   & 2.325 MiB   \\
  \midrule
  Bandwidth usage & 8.187 KiB & 289.8 KiB & 854.6 KiB & 1.722 MiB & 2.435 MiB \\
  \midrule
  Query time    & 0.5024 ms  & 1.9587 ms   & 5.3174 ms   & 10.5765 ms  & 15.3573 ms  \\
  Answer time   & 0.0221 ms  & 12.9588 ms  & 69.1425 ms  & 1447.86 ms  & 2816.19 ms  \\
  Recovery time & 0.574 ms   & 13.247 ms   & 42.825 ms   & 169.721 ms  & 807.623 ms  \\
  \midrule
  Total online time  & 1.09 ms  & 28.16 ms  & 117.28 ms  & 1628.15 ms & 3639.17 ms \\
  \bottomrule
  \end{tabular}
  }
\end{table}

\subsection{Manipulation Difficulty Analysis}
\label{sec:manip}

We now establish probabilistic guarantees that hiding the audit set through \ac{respir} increases the difficulty of targeted manipulations. 
The results formalize a core intuition: if the provider cannot identify the audited queries, then passing the audit with high probability requires improving behavior on a nontrivial fraction of the entire candidate set, rather than selectively patching a known evaluation subset.

\begin{theorem}[\ac{respir} audit manipulation]\label{thm:n_manipulation_respir_audit}
Let $C \subset \mathcal{X}$ be the candidate set (with $\lvert C \rvert = N$) and let $S \subset C$ be the audit set (with $\lvert S \rvert = n$), sampled uniformly at random without replacement from $C$.
Let \(d_{C,\text{true}}\) denote the demographic parity gap of the provider's model on \(C\) (in the absence of manipulation), let \(n_{\min} = \min(n_0,n_1)\) be the size of the smallest protected group in \(S\), let \(N_{\min} = \min(N_0,N_1)\) be the size of the smallest protected group in \(C\), and let \(\mtol\) be the audit threshold.
For $\delta \in (0, \frac{1}{2}]$, assuming the model would fail the audit, \emph{i.e.}, \(|d_{C,\text{true}}| > \mtol\), it suffices for the provider to flip at least
\begin{equation}
m_\text{\ac{respir}} \;\geq\; \left\lceil \left(
   \lvert d_{C,\text{true}} \rvert - \mtol 
    + \sqrt{\frac{2\ln(4/\delta)}{n_{\min}}}
\right) N_{\min} \right\rceil \label{eq:n_manipulation_respir_audit}
\end{equation}
predictions in $C$ in order to pass the audit with probability at least $1 - \delta$.
\end{theorem}

\begin{figure*}[tb!]
    \centering  
     \begin{subfigure}[t]{\linewidth}
     \centering
         \includegraphics[width=0.4\linewidth]{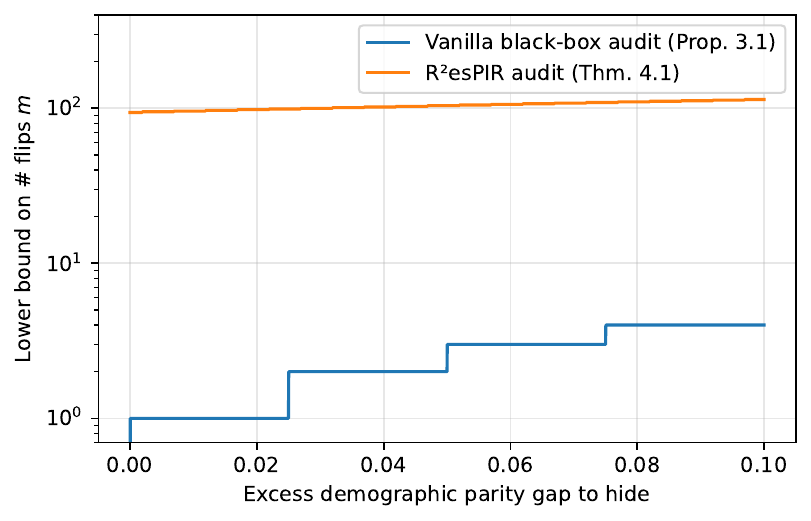}
         \includegraphics[width=0.4\linewidth]{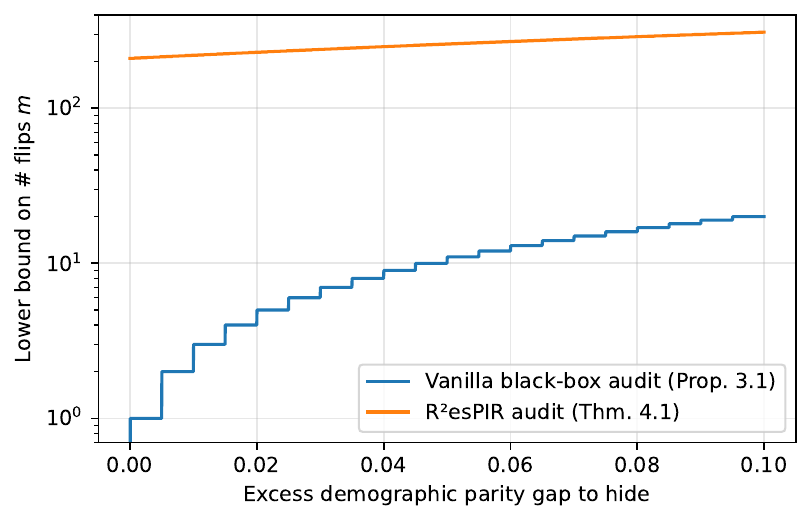}
         \caption{Lower bounds on the number of output flips $m$ required to hide unfairness, for the unbalanced setting (left) and the balanced setting (right).}
         \label{subfig:manipulation_cost_illustration}
     \end{subfigure}

     \begin{subfigure}[t]{\linewidth}
     \centering
         \includegraphics[width=0.4\linewidth]{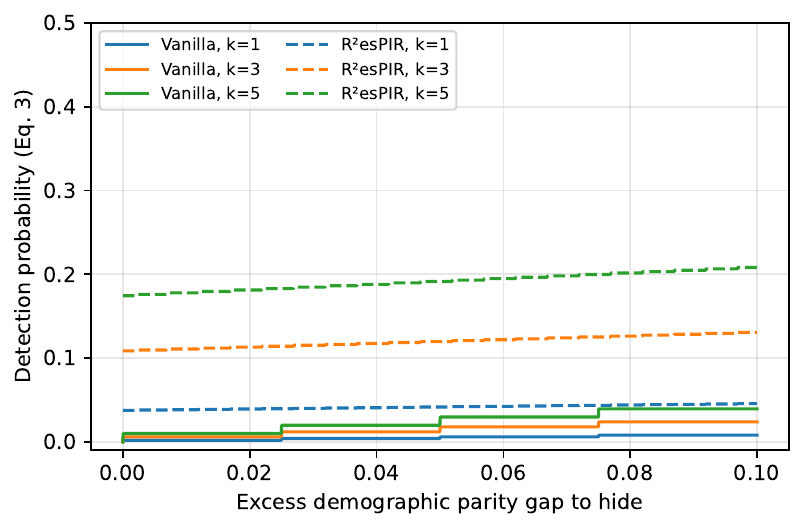}
         \includegraphics[width=0.4\linewidth]{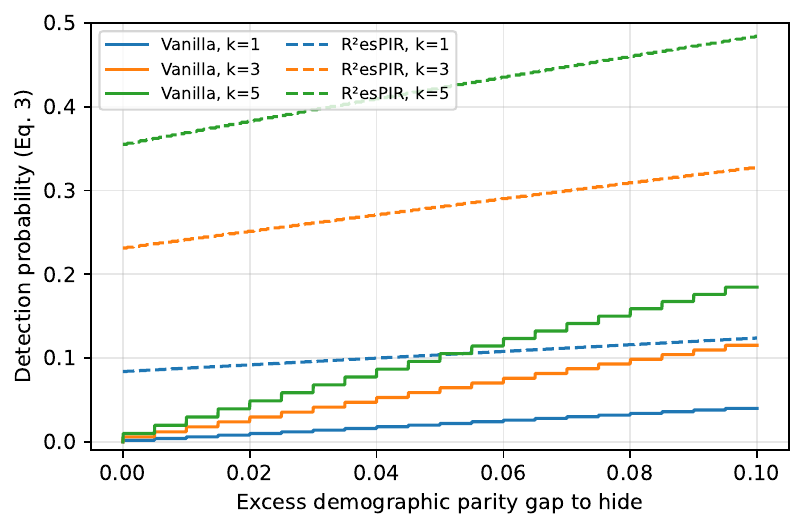}
         \caption{Probability that the auditor detects manipulation, for the unbalanced setting (left) and the balanced setting (right).}
         \label{subfig:detection_probability_illustration}
     \end{subfigure}

    \caption{Illustration of our theoretical bounds on manipulation difficulty for a vanilla black-box audit and for \ac{respir}. 
    We consider audit sets of size $|S|=400$ and $k\in\{1,3,5\}$ canaries, each detected with effectiveness probability $q=0.8$. 
    For \ac{respir}, we use a candidate set of size $|C|=2{,}000$ and target passing probability $1-\delta=0.95$. 
    We compare two protected group-composition regimes: an unbalanced regime where the minority group represents $10\%$ of $S$ (and of $C$ for \ac{respir}) and a balanced regime with equal group sizes. 
    The horizontal axis varies the amount of unfairness above the tolerance $\varepsilon$.}
    \label{fig:bounds_illustration}
\end{figure*}

\begin{proof}
Let $\gamma := \sqrt{\frac{2\ln(4/\delta)}{n_{\min}}}$.
By Theorem~\ref{cor:manip_sufficient_condition} (provided hereafter in Section~\ref{subsubsec:intermediate_th}), a sufficient condition for passing the audit with probability at least $1-\delta$ (conditional on $n_0,n_1$) is that the (post-manipulation) candidate-set gap satisfies
\[
|d_C|\le \mtol-\gamma.
\]
Thus, it suffices for the provider to modify predictions on $C$ so as to reduce the initial gap $|d_{C,\text{true}}|$ down to at most $\mtol-\gamma$.

We now bound the number of prediction flips required to achieve such a reduction on $C$.
Assume without loss of generality that $d_{C,\text{true}}>0$ (otherwise apply the same argument to $-d_{C,\text{true}}$).
Consider a single flip of one prediction on $C$.
If the flipped example belongs to group $A=1$, then $p_{1,C}$ changes by at most $1/N_1$, hence $d_C$ changes by at most $1/N_1$.
If it belongs to group $A=0$, then $p_{0,C}$ changes by at most $1/N_0$, hence $d_C$ changes by at most $1/N_0$.
Therefore, a single flip can change $d_C$ by at most
\[
\max\!\left(\frac{1}{N_0},\frac{1}{N_1}\right)
=
\frac{1}{N_{\min}}.
\]
To reach $d_C\le \mtol-\gamma$, the provider must decrease the gap by at least
\[
d_{C,\text{true}}-(\mtol-\gamma)
=
|d_{C,\text{true}}|-\mtol+\gamma.
\]
Since each flip can reduce $d_C$ by at most $1/N_{\min}$, it suffices to flip at least
\[
\left\lceil \left(|d_{C,\text{true}}|-\mtol+\gamma\right)N_{\min}\right\rceil
\]
predictions in $C$, as claimed.
\end{proof}

The manipulation resistance of \ac{respir} (\Cref{thm:n_manipulation_respir_audit}) improves on that of a ``vanilla'' black-box audit (\Cref{thm:n_manipulation_vanilla_audit}) in two complementary ways.
First, \ac{respir} enforces a stricter effective target on the provider’s underlying disparity because of the sampling uncertainty term $\sqrt{2\ln(4/\delta)/n_{\min}}$. 
This term materializes the fact that under uniform random sampling from $C$, audits sets $S$ might exhibit unfairness $d_S$ slightly lower or higher due to the sampling process. 
Nevertheless, as long as they are representative of the underlying distribution (which is the case unless $n = \lvert S \rvert$ is pathologically small), the provider's model is expected to behave fairly on them as well. 
In other words, to ensure unfairness does not exceed $\varepsilon$ on the sampled $S$ with high probability, the provider must over-correct unfairness over the candidate set.
Second, \ac{respir} increases the scale at which modifications must be performed. 
In a vanilla black-box audit, the provider may concentrate changes on the audited set $S$, so the best-case manipulation complexity scales with $n_{\min}$.
Under \ac{respir}, the provider must instead improve fairness across the entire candidate set $C$ to pass for an unknown random $S$, implying a manipulation complexity that scales with $N_{\min}$, the size of the smaller protected group in $C$.
Since typically $N_{\min}\gg n_{\min}$, this yields a substantially larger manipulation cost.

Taken together, these two effects force the provider to modify a substantially larger number of outputs to pass reliably, which in turn increases the probability that manipulation is detected. 
Indeed, given a verification set of $k$ canaries (drawn from $C$ and externally labeled), the probability of detection increases with the manipulated fraction $m/N$ (according to~\eqref{eq:manipulation_detection_proba}, in which we substitute $n$ by $N$ to model the probability that a canary's output is modified). 
Moreover, when $k$ is small compared to the audit size $n$, the auditor can include these canaries alongside the sampled audit set $S$ without materially affecting the fairness estimation guarantees derived above.

\paragraph{Illustrative Example}
The manipulation difficulty of both the vanilla black-box audit and our proposed \ac{respir} audit is illustrated in \Cref{fig:bounds_illustration} for a representative parameter setting. 
In particular, \Cref{subfig:manipulation_cost_illustration} reports the number of output flips the provider must perform as a function of the amount of unfairness to hide, while \Cref{subfig:detection_probability_illustration} shows the corresponding probability that the auditor detects such manipulation for different numbers of canaries included in the audit set~$S$. 
As expected, protected group imbalance (left panels) substantially facilitates audit manipulation compared to the balanced case (right panels), as fewer output flips are needed to drive the observed disparity below the tolerance~$\mtol$. 
In both regimes, \ac{respir} notably increases the required number of flips~$m$ and, consequently, the probability of manipulation detection.

\subsubsection{Intermediate Concentration Result}\label{subsubsec:intermediate_th}

Theorem~\ref{thm:n_manipulation_respir_audit} follows from Theorem~\ref{cor:manip_sufficient_condition}, which formalizes the fact that the audit only observes the disparity through the sample-based estimate $d_S$. When $S$ is sampled uniformly at random from $C$, $d_S$ concentrates around $d_C$, with deviations controlled at scale $\sqrt{\ln(1/\delta)/n_{\min}}$. 
Consequently, to pass with high probability, the provider must keep its candidate-set gap at least this far \emph{below} the fairness threshold.

\begin{theorem}[\ac{respir} audit: Sufficient candidate set condition for high-probability passing]\label{cor:manip_sufficient_condition}
Let $n_{\min} := \min(n_0,n_1)$ and fix $\delta\in(0,1)$. 
If the provider's candidate set gap satisfies
\[
|d_C|
\;\le\;
\mtol 
\;-\;
\sqrt{\frac{2\ln(4/\delta)}{n_{\min}}},
\]
then the provider passes the audit with probability at least $1-\delta$:
\[
\Pr(|d_S|\le \mtol \mid n_0,n_1)\;\ge\;1-\delta.
\]
\end{theorem}
\begin{proof}
Conditioning on \((n_0,n_1)\), we apply finite-population Hoeffding/Serfling
bounds~\cite{hoeffding1963,serfling1974} to each group rate \(p_{a,S}\),
and then use the triangle inequality and a union bound to obtain
\[
\Pr\!\left(|d_S-d_C|\ge \gamma \mid n_0,n_1\right)
\le 4\exp\!\left(-\frac{\gamma^2}{2}n_{\min}\right).
\]
With \(\gamma=\sqrt{2\ln(4/\delta)/n_{\min}}\), this probability is at most
\(\delta\). Hence, if \(|d_C|\le \mtol-\gamma\), then \(|d_S|\le \mtol\)
with probability at least \(1-\delta\). See Appendix~\ref{app:proof_manip_sufficient_condition} for the detailed proof.
\end{proof}

\begin{table*}[h!]
\centering
\caption{Overview of the characteristics of the considered datasets.}
\label{tab:datasets}
\resizebox{0.7\linewidth}{!}{
\begin{tabular}{llcccccccc}
\hline
\textbf{Dataset} & \textbf{\begin{tabular}[c]{@{}c@{}}Protected\\ Attribute\end{tabular}} & \textbf{\begin{tabular}[c]{@{}c@{}}Dataset\\ Size\end{tabular}} & $Acc$ & \textbf{$|d_{C,\text{true}}|$} & \textbf{$N$} & \textbf{$n$} & $N_\text{min}$ & $n_\text{min}$ & \textbf{$\frac{N_{\text{min}}}{N}$} \\
\hline
\multirow{4}{*}{CCD} & gender &\multirow{4}{*}{30,000} & 0.823 & 0.028 & \multirow{4}{*}{13,500} & \multirow{4}{*}{6,750} & 5,350 & 2,675 & 0.396 \\
 & education & & 0.82 & 0.044 &  &  & 2,417 & 1,209 & 0.179 \\
 & marriage &  & 0.821 & 0.009 &  &  & 6,147 & 3,073 & 0.455 \\
 & age &  & 0.82 & 0.027 &  &  & 4,956 & 2,478 & 0.367 \\
\hline
\multirow{2}{*}{COMPAS} & gender & \multirow{2}{*}{6,172} & 0.66 & 0.148 & \multirow{2}{*}{2,778} & \multirow{2}{*}{1,389} & 529 & 264 & 0.19 \\
 & race &  & 0.67 & 0.164 &  &  & 947 & 473 & 0.341 \\
\hline
\multirow{1}{*}{hateday} & language & 81,000 & 0.95 & 0.008 & 36,450 & 18,225 & 3,038 & 1,519 & 0.083 \\
\hline
\end{tabular}
}
\end{table*}

\section{Experimental Evaluation}\label{sec:expes}

To illustrate and empirically validate our results, we selected three datasets spanning tabular and text domains: Default of Credit Card Clients (CCD)~\cite{Yeh2009TheCO}, COMPAS~\cite{larsonHowWeAnalyzed2016}, and HateDay~\cite{tonneauHateDayInsightsGlobal2025}. 
\revision{
We extracted numerical features from the HateDay text data using an off-the-shelf sentence transformer\footnote{\url{huggingface.co/sentence-transformers/distiluse-base-multilingual-cased-v2}}.
For a fixed protected attribute, each dataset was randomly split $70/30$ into a training set and a candidate set, using stratification with respect to the sensitive attribute. 
The model and data pre-processing is a histogram gradient boosted tree baseline from the \texttt{skrub} library\footnote{\url{skrub-data.org/stable/reference/generated/skrub.tabular_pipeline.html}}.
We tuned the learning rate, number of leaf nodes and number of estimators via random search with $5$-fold cross-validation on the training set.
}
After hyperparameter selection, we retrained the model on the full training set. 
Our source code is available in our online repository\footnote{\url{github.com/sofianeazogagh/oblivious_audit}}, which includes all scripts and datasets required to reproduce our experiments and results, as well as our PIR-based auditing framework implemented as a user-friendly module.

\Cref{tab:datasets} summarizes the main characteristics of each dataset. 
Specifically, for each dataset we report the model accuracy ($Acc$) and the demographic parity gap ($|d_{C,\text{true}}|$) before any manipulation, both measured on the candidate set $C$. 
We also report the candidate set size $N$, the audit set size $n$ as well as the size of the smallest protected group in each set (respectively $N_{\min}$ and $n_{\min}$). 
As shown in the last column, some setups are highly imbalanced. 
For example, in HateDay the smallest protected group represents only about $8\%$ of the candidate set.

We then simulate audit manipulation, in which a malicious provider attempts to hide half of its unfairness. 
That is, for each dataset and protected attribute, we set $\varepsilon = |d_{C,\text{true}}|/2$. 
We leverage \Cref{thm:n_manipulation_vanilla_audit} and~\Cref{thm:n_manipulation_respir_audit} to compute the number of output modifications required to hide this amount of unfairness against a traditional (``vanilla'') black-box audit and against a \ac{respir}-based audit. 
For \ac{respir}, the provider targets passing probability $1-\delta = 80\%$. 
Finally, we plug these modification counts into \Cref{eq:manipulation_detection_proba} to obtain the corresponding probability of manipulation detection in a setting where the auditor has a verification set of $k \in \{5, 10, 20, 50\}$ canaries (see \Cref{sec:manip}). 
Note that Theorem~\ref{thm:n_manipulation_respir_audit} should be interpreted as an analytical characterization of manipulation difficulty, rather than as an operational recipe. 
Indeed, the bound depends on quantities that are not all known to the auditor, such as the platform's true candidate-set disparity, and, under random sampling, on quantities not known in advance to the platform, such as the realized value of \(n_{\min}\).

The results are summarized in Table~\ref{tab:main_results}, \revision{and displayed in Figure~\ref{fig:results} for a subset of our experiments}. 
As expected, the number of output manipulations required to hide unfairness increases substantially under \ac{respir} compared to a traditional (``vanilla'') black-box audit. 
This, in turn, leads to systematically higher probabilities of detecting such manipulation, often by roughly a factor of two, which is particularly noteworthy since the final audit set size $n = |S|$ is identical for both audit frameworks.
Comparing results across datasets and protected attributes and relating them to the dataset characteristics summarized in Table~\ref{tab:datasets}, two main factors drive detectability. 
First, protected group imbalance greatly facilitates manipulation. 
For instance, for the highly imbalanced HateDay setting, the probability of detecting manipulation remains very small even with $k=50$ canaries. 
Second, larger unfairness levels to be hidden require more output flips, which increases detectability. 
For example, the highest detection probabilities for both audit frameworks occur for COMPAS with race as the protected attribute, which is also the configuration with the largest initial unfairness.
\revision{As shown in Figure~\ref{fig:results}, the number of canaries owned by the auditor remains a crucial parameter, strongly influencing manipulation detectability for both methods.}

 \begin{table}[t!]
 \centering
 \caption{Manipulation cost and manipulation detection probability for each dataset and protected attribute. 
 For each setting, we report the required number of output flips to hide half of the initial unfairness under a vanilla black-box audit and under \ac{respir}, together with the resulting detection probabilities for different numbers of canaries $k$.}
 \label{tab:main_results}
 \scriptsize
\setlength{\tabcolsep}{2.5pt}
\resizebox{0.92\linewidth}{!}{
\begin{tabular}{m{1cm}m{1cm}m{0.65cm}m{0.95cm}m{0.50cm}m{0.9cm}m{1.1cm}}
\hline
\textbf{Dataset}         & \textbf{\begin{tabular}[c]{@{}l@{}}Protected\\ Attribute\end{tabular}} & \textbf{$m_{\text{vanilla}}$} & \textbf{$m_{\text{\ac{respir}}}$} & \textbf{$k$} & \textbf{$P_{\text{detect}}$ (vanilla)} & \textbf{$P_{\text{detect}}$ (\ac{respir})} \\ \hline
\multirow{4}{*}{CCD}     & \multirow{4}{*}{gender}                                                & \multirow{4}{*}{38}           & \multirow{4}{*}{152}              & 5            & 2.23\%                                 & 4.42\%                                     \\
                         &                                                                        &                               &                                   & 10           & 4.41\%                                 & 8.65\%                                     \\
                         &                                                                        &                               &                                   & 20           & 8.63\%                                 & 16.55\%                                    \\
                         &                                                                        &                               &                                   & 50           & 20.20\%                                & 36.39\%                                    \\ \hline
\multirow{4}{*}{CCD}     & \multirow{4}{*}{education}                                             & \multirow{4}{*}{27}           & \multirow{4}{*}{107}              & 5            & 1.59\%                                 & 3.13\%                                     \\
                         &                                                                        &                               &                                   & 10           & 3.15\%                                 & 6.16\%                                     \\
                         &                                                                        &                               &                                   & 20           & 6.21\%                                 & 11.95\%                                    \\
                         &                                                                        &                               &                                   & 50           & 14.81\%                                & 27.24\%                                    \\ \hline
\multirow{4}{*}{CCD}     & \multirow{4}{*}{marriage}                                              & \multirow{4}{*}{14}           & \multirow{4}{*}{56}               & 5            & 0.83\%                                 & 1.65\%                                     \\
                         &                                                                        &                               &                                   & 10           & 1.65\%                                 & 3.27\%                                     \\
                         &                                                                        &                               &                                   & 20           & 3.27\%                                 & 6.43\%                                     \\
                         &                                                                        &                               &                                   & 50           & 7.97\%                                 & 15.31\%                                    \\ \hline
\multirow{4}{*}{CCD}     & \multirow{4}{*}{age}                                                   & \multirow{4}{*}{34}           & \multirow{4}{*}{134}              & 5            & 2.00\%                                 & 3.91\%                                     \\
                         &                                                                        &                               &                                   & 10           & 3.96\%                                 & 7.66\%                                     \\
                         &                                                                        &                               &                                   & 20           & 7.76\%                                 & 14.74\%                                    \\
                         &                                                                        &                               &                                   & 50           & 18.28\%                                & 32.88\%                                    \\ \hline
\multirow{4}{*}{COMPAS}  & \multirow{4}{*}{gender}                                                & \multirow{4}{*}{20}           & \multirow{4}{*}{79}               & 5            & 5.63\%                                 & 10.87\%                                    \\
                         &                                                                        &                               &                                   & 10           & 10.94\%                                & 20.56\%                                    \\
                         &                                                                        &                               &                                   & 20           & 20.68\%                                & 36.89\%                                    \\
                         &                                                                        &                               &                                   & 50           & 43.97\%                                & 68.36\%                                    \\ \hline
\multirow{4}{*}{COMPAS}  & \multirow{4}{*}{race}                                                  & \multirow{4}{*}{39}           & \multirow{4}{*}{156}              & 5            & 10.74\%                                & 20.53\%                                    \\
                         &                                                                        &                               &                                   & 10           & 20.32\%                                & 36.85\%                                    \\
                         &                                                                        &                               &                                   & 20           & 36.51\%                                & 60.12\%                                    \\
                         &                                                                        &                               &                                   & 50           & 67.89\%                                & 89.96\%                                    \\ \hline
\multirow{4}{*}{hateday} & \multirow{4}{*}{language}                                              & \multirow{4}{*}{7}            & \multirow{4}{*}{26}               & 5            & 0.15\%                                 & 0.28\%                                     \\
                         &                                                                        &                               &                                   & 10           & 0.31\%                                 & 0.57\%                                     \\
                         &                                                                        &                               &                                   & 20           & 0.61\%                                 & 1.14\%                                     \\
                         &                                                                        &                               &                                   & 50           & 1.52\%                                 & 2.81\%                                     \\ \hline
\end{tabular}
}

 \end{table}
 
\begin{figure}
\centering
     \includegraphics[width=0.37\textwidth]{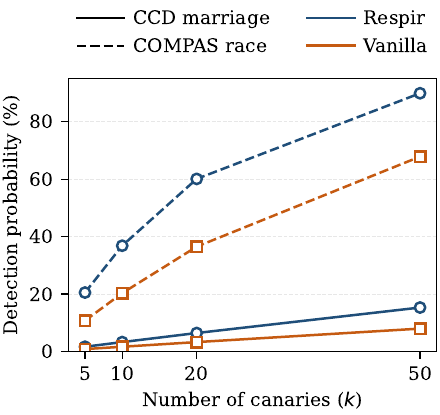}
     \caption{\revision{Probability of manipulation detection for a subset of our experiments, as a function of the number of canaries $k$, for vanilla audits and for our proposed \acs{respir} audit framework.}}
     \label{fig:results}
 \end{figure}
 
It is also worth noting that when the tolerance $\mtol$ is close to zero, and in particular when it is smaller than $\gamma = \sqrt{\frac{2\ln(4/\delta)}{n_{\min}}}$ (which can occur, for instance, when $n_{\min}$ is small), the sufficient target
\[
|d_C| \;\le\; \mtol \;-\; \sqrt{\frac{2\ln(4/\delta)}{n_{\min}}}
\]
for passing with probability at least $1-\delta$ might become unattainable, since $|d_C|\ge 0$. 
In our experiments, such cases arise for some configurations in which $|d_{C,\text{true}}|$ is already very small. 
In these regimes, the provider may be unable to reach the desired passing confidence level $1-\delta$ for the \ac{respir} audit. Nevertheless, we still report its best-effort manipulation, namely the number of output modifications required to achieve $|d_C|=0$, which, as reflected in our results, still tends to make manipulation more detectable by the auditor.

Overall, these findings confirm the purpose of \ac{respir}. 
By hiding the audited set $S$ through PIR, our protocol forces malicious providers to spread manipulations over a larger fraction of the candidate set, thereby increasing the effort required to hide unfairness and substantially improving the probability of detecting such manipulation.

\section{Limitations, Discussion and Conclusion}
\label{sec:conclusion}

\subsection{\revision{Applicability of the Proposed Approach}}
\label{sec:applicability}

\paragraph{\revision{Audited Model.}} \revision{
Our proposed framework, as well as the accompanying theoretical analysis, is agnostic to the type of audited model. 
Indeed, the provider’s model class matters only if the provider must realize flips by retraining a new classifier within the same hypothesis space. 
Our adversary model is more powerful as the provider can patch the deployed system by overriding outputs on selected inputs without retraining. 
This is at the same time realistic but also a worst-case situation for the auditor. 
If the provider was not able to flip selected answers without a full model retraining, arbitrary flips might be infeasible and the provider could need more changes to hide unfairness, so our lower bounds (Proposition~\ref{thm:n_manipulation_vanilla_audit} and Theorem~\ref{thm:n_manipulation_respir_audit}, holding for any input-output mapping) would remain valid but potentially conservative.
}

\paragraph{\revision{Audited Metric.}}\revision{While in this paper, we focus on the demographic parity metric, our approach naturally extends to auditing other model properties. Indeed, the PIR-based component treats the fairness computation as a modular measurement block, so our proposed auditing framework is metric-agnostic.
In particular, Proposition~\ref{thm:n_manipulation_vanilla_audit} and Theorem~\ref{thm:n_manipulation_respir_audit} can be straightforwardly adapted to other fairness criteria (including separation- and sufficiency-based metrics, which encompasses but is not limited to group fairness metrics based on true positive rates (equal opportunity), false positive rates (predictive equality) or both (equalized odds)), by incorporating the proper sensitivity value suited to the audited measure. Similar guarantees regarding resilience to manipulation would then hold, as illustrated for the equal opportunity metric in Appendix~\ref{app:equal-opportunity}.
On the same line, our framework and guarantees can also be applied to setups with several sensitive attributes or multi-class problems, as long as the sensitivity of the audited function can be analytically expressed or bounded (as in the proofs of Proposition~\ref{thm:n_manipulation_vanilla_audit} and Theorem~\ref{thm:n_manipulation_respir_audit}), key to lower bounding the number of flips required to fool the audit. 
We therefore expect the same modularity could support regression as well, though fairness for regression has been less explored in the literature.}

\subsection{\revision{Discussion}}

We studied the problem of auditing fairness claims in adversarial settings, in which an audited provider may adapt its outputs once it can anticipate which inputs will be inspected. 
Our main contribution is \acs{respir}, an oblivious auditing protocol that leverages PIR to hide the auditor’s queries within a larger candidate set. 
By construction, this prevents the provider from patching a known audit set.
Instead, passing with high probability requires spreading modifications over a nontrivial fraction of the candidate set. 
We formalized this intuition through concentration-based guarantees that lower bound the manipulation effort needed to pass, and we combined these bounds with a canary-based verification mechanism to translate manipulation effort into a detection probability. 
Our experiments on tabular and text datasets confirms that for a fixed audit set size, \acs{respir} increases the number of required output flips and, consequently, improves the detectability of targeted manipulation, with the strongest gains arising when the candidate set is large and sufficiently balanced across protected groups.

Nevertheless, as illustrated by our experiments, strong protected group imbalance and small audit sets remain limiting factors. 
The manipulation guarantees in \Cref{sec:manip} and their empirical evaluation in \Cref{sec:expes} highlight three main design principles:
\begin{enumerate}
\item \emph{Large and representative candidate sets.}
As the best-case manipulation cost scales with $\min(N_0,N_1)$, the candidate set should be large and should contain nontrivial representation of each protected group.
\item \emph{Balanced auditing.}
Since detection power scales with $\min(n_0,n_1)$, the auditor should sample $S$ using a rule that avoids severe group imbalance; stratified sampling from $C$ is a natural option.
\item \emph{Hiding the audit set is necessary but not always sufficient.}
PIR prevents the provider from learning \emph{which} queries are in the audit set, blocking the cheapest form of manipulation (patching exactly $S$, as in \Cref{thm:n_manipulation_vanilla_audit}). 
However, our guarantees rely on the candidate set capturing a meaningful portion of the provider’s operational domain. 
If $C$ is not representative, then passing the audit may not translate into improved fairness outside~$C$.
\end{enumerate}

\paragraph{Limitations.}
We view \acs{respir} as a proof-of-concept showing how cryptographic primitives can be leveraged to hide the actual queries used during an algorithmic audit. 
Its main benefit is to prevent audit set adaptivity: the platform can no longer tailor its responses to the particular audit set \(S\), but must instead commit to outputs over a larger candidate set \(C\). 
This makes manipulation more costly and, when combined with external verification queries, more detectable. 
Nonetheless, \acs{respir} also has inherent limitations.

First, because the candidate set \(C\) is known to both the platform and the auditor, the protocol does not prevent all forms of manipulation. 
A strategic platform can still attempt to modify its outputs over \(C\) so as to satisfy the audited fairness criterion with high probability. 
A stronger alternative would be to require a private-inference API, for instance based on verifiable homomorphic encryption or on hybrid cryptographic protocols, allowing the auditor to evaluate the commited model on private inputs without revealing the audited queries to the platform. 
Conceptually, this would remove the need for an explicit shared candidate set, since the auditor's query domain would no longer have to be enumerated in advance. 
In practice, however, such a requirement remains restrictive.
Encrypted inference typically incurs substantial computational and communication overhead, and often requires model-specific adaptations such as polynomial approximations, quantization, or hybrid protocols for non-linear operations~\citep{giladbachrach2016cryptonets,brutzkus2019lola,juvekar2018gazelle,mishra2020delphi,gong2024practical}, not to mention the verification step which introduces additional overhead. 
Therefore, we view the use of a fixed candidate set as a pragmatic compromise between limiting the platform's computational burden, ensuring deployment feasibility, and preserving audit usefulness.

Second, manipulation detectability could be further enhanced by using an iterative protocol, in which the auditor progressively requires the platform to label the candidate set \(C\), in an order designed to trigger its worst-case manipulation bounds. 
For instance, the auditor could first require the platform to label the smallest protected group, thereby constraining the platform's subsequent manipulation strategy and potentially forcing its final flips to be performed on the largest subgroup, resulting in more changes and higher detectability. 
This would prevent the platform from reaching its lower-cost manipulation bounds, making manipulation more difficult in practice. 
However, such an extension would need to be designed carefully to preserve the non-adaptivity property of \acs{respir}; otherwise, the iterative structure itself could reveal information that the platform may exploit strategically.

\subsection{\revision{Conclusion}}

Overall, \acs{respir} is a step toward manipulation-resistant audits combining cryptographic query hiding with probabilistic guarantees. 
We hope this work will motivate further research at the intersection of auditing, cryptography and trustworthy machine learning, and inform the design of regulatory audits robust to adversarial environments.  

Several directions remain open. A first avenue is to extend the analysis beyond demographic parity to separation- and sufficiency-based criteria, as well as to settings with continuous outputs or multi-class protected attributes. 
A second direction is to study candidate-set design in greater depth, including how to construct representative candidate sets under limited data access, how to allocate audit budgets across groups to mitigate imbalance, and how to choose the confidence parameter to balance robustness and audit cost. 
Finally, from a systems perspective, deploying \acs{respir} in real auditing pipelines calls for exploring the performance trade-offs of different PIR instantiations, robust engineering under rate limits and latency constraints and protocol mechanisms for logging and accountability that remain compatible with query privacy.

\section*{Acknowledgments}
This research was enabled in part by funding from the SCALE-AI Chair in Data-Driven Supply Chains as well as by the \emph{Fonds de recherche du Québec} -- \emph{Nature et technologies (FRQNT)} through a Team Research Project \emph{(327090)}. Sébastien Gambs is supported by the Canada Research Chair program (CRC on Privacy-preserving and Ethical Analysis of Big Data) and a Discovery Grant from NSERC. The authors would like to thank the anonymous reviewers for their valuable suggestions.

\bibliography{pir,sample-base}

\onecolumn

\appendix
\section{Interactions during the PIR}\label{appendix:pir_scheme}

Figure~\ref{fig:verisimplepir} provides a detailed description of the interaction during the PIR between the auditor and the provider using VeriSimplePIR~\cite{verisimplepir24}.

\begin{figure*}[h!]
    \centering
    \begin{tikzpicture}[font=\small]
    
    \def\W{0.7\textwidth}
    \def\pad{6pt}
    
    \node[draw, inner sep=\pad, line width=0.8pt] (FRAME) {
    \begin{minipage}{\W}
    
    \begin{center}
    \begin{tabular*}{\W}{@{\extracolsep{\fill}}ccc}
    \underline{\textbf{Auditor}}
    &
    &
    \underline{\textbf{Provider}}
    \end{tabular*}
    \end{center}
    
    \vspace{2mm}
    \hrule
    \vspace{2mm}

    \begin{tabular*}{\W}{@{\extracolsep{\fill}}lll}
    Request an audit & &
    \end{tabular*}
    
    \begin{center}
    \underline{\textbf{Offline Commitment}}
    \end{center}
    
    \vspace{1mm}
    
    \begin{tabular*}{\W}{@{\extracolsep{\fill}}lll}
                     &                              & $\D \in \Z_p^{\ell \times \mu}$ \\[0.5mm] 
                     &                              & Uniformly sample $A \in \Z_q^{\mu\times n}$ \\[0.5mm] 
                     &                              & $H \gets \D \cdot A$ \\[0.5mm]
                     &                              & $C \gets \Hash(A,H)$ \\[0.5mm]
                     & \hspace{1.3cm}$\xleftarrow{\qquad A,H,Z \qquad}$ & $Z \gets C \cdot \D$ \\[0.5mm]
    $C \gets \Hash(A,H)$ &  &
    \end{tabular*}
    
    \vspace{1mm}
    \begin{center}
    Return \texttt{Fail} if $||Z||_{\infty} > p\cdot \ell$ or $ZA \neq CH$
    \end{center}
    
    \vspace{2mm}
    \hrule
    \vspace{2mm}
    
    \begin{center}
    \underline{\textbf{Query}}
    \end{center}
    
    \vspace{1mm}
    
    \begin{tabular*}{\W}{@{\extracolsep{\fill}}lll}
    Choose an index $i=(i_r,i_c) \in \Z_\ell \times \Z_\mu$ &                   & \\[0.5mm]
    $b_{i_c} \gets \OHE(i_c)$  &                   & \\[0.5mm]
    \Comment{Regev's Encryption using the generated A during the commitment}  &                   & \\[0.5mm]
    Sample a uniform secret $s \in \Z_q^n$  &                   & \\[0.5mm]
    Sample an error vector $e \in \chi^\mu$  &                   & \\[0.5mm]
    $\enc{b_{i_c}} = As + e + \lfloor \frac{q}{p} \rfloor \cdot b_{i_c} $ & & \\[0.5mm]
     & \hspace{-5.3cm}$\xrightarrow{\qquad \enc{b_{i_c}} \qquad}$ & \\
    \end{tabular*}

    \vspace{2mm}
    \hrule
    \vspace{2mm}

    \begin{center}
    \underline{\textbf{Answer}}
    \end{center}

    \vspace{3mm}

    \begin{tabular*}{\W}{@{\extracolsep{\fill}}lll}
    &  & $v \gets \D \times \enc{b_{i_c}}$ \\[0.5mm]
    & \hspace{1.7cm}$\xleftarrow{\qquad \enc{r} \qquad}$ & \\
    \end{tabular*}

    \vspace{2mm}
    \hrule
    \vspace{2mm}

    \begin{center}
    \underline{\textbf{Recover}}
    \end{center}

    \vspace{3mm}

    \begin{tabular*}{\W}{@{\extracolsep{\fill}}lll}
    $\tilde{r} \gets v[i_r] - \langle H[i_r],s \rangle$ &  &  \\[0.5mm]
    $r \gets \lfloor\frac{p\cdot \tilde{r}}{q} \rceil$ &  &  \\[0.5mm]
    \end{tabular*}

    \vspace{5mm}

    \end{minipage}
    };
    
    \end{tikzpicture}
    \caption{The interaction between the auditor and the provider based on the SimplePIR protocol~\cite{simplepir23} and using the commitment to the database from VeriSimplePIR~\cite{verisimplepir24}. More details can be found in \cite{verisimplepir24} (Construction 5.1) to verify the integrity of the queries.}
    \label{fig:verisimplepir}
\end{figure*}
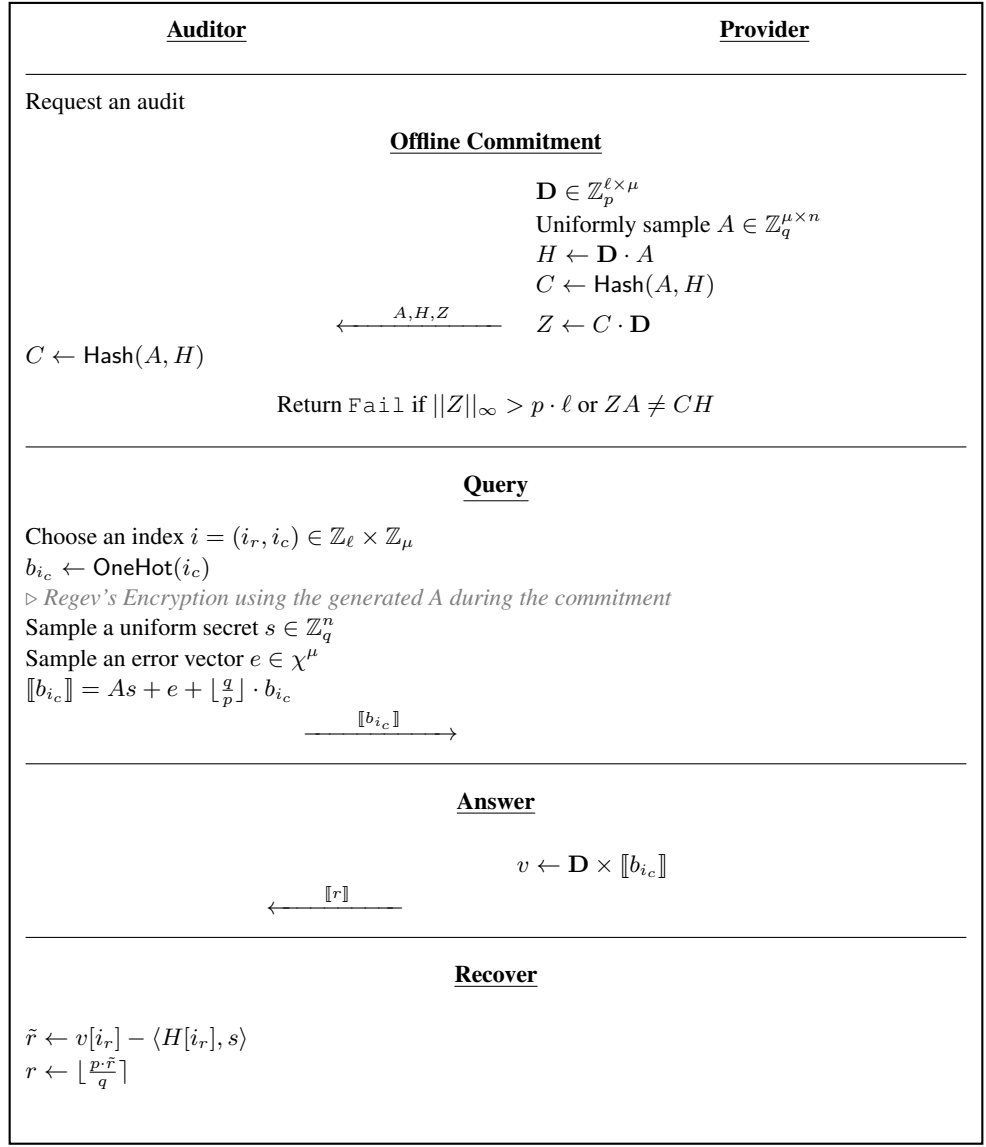

\section{Proof of Theorem~\ref{cor:manip_sufficient_condition}} \label{app:proof_manip_sufficient_condition}

\begin{proof}
Fix a realization of $(n_0,n_1)$ and let $n_{\min}:=\min(n_0,n_1)$.
Recall that $d_C = p_{1,C}-p_{0,C}$ and $d_S = p_{1,S}-p_{0,S}$, in which $p_{a,C}$ (resp.\ $p_{a,S}$) denotes the protected group-$a$ positive prediction rate on $C$ (resp.\ on $S$).

Conditioning on $n_a$, the quantity $p_{a,S}$ is the average of $n_a$ binary values in $[0,1]$ obtained by sampling $n_a$ elements uniformly without replacement from the $N_a$ elements of $C$ in group $a$.
For i.i.d.\ sampling with replacement, Hoeffding's inequality for bounded random variables yields, for any $t>0$,
\[
\Pr\!\left(\left|p_{a,S}-p_{a,C}\right|\ge t \,\middle|\, n_a\right)
\;\le\;
2\exp(-2n_a t^2).
\]
The same exponential bound remains valid under uniform sampling without replacement~\citep{hoeffding1963}, so we use it here.

Using the decomposition
\[
d_S-d_C
=
(p_{1,S}-p_{1,C})-(p_{0,S}-p_{0,C}),
\]
and the triangle inequality,
\[
|d_S-d_C|
\le
|p_{1,S}-p_{1,C}|+|p_{0,S}-p_{0,C}|,
\]
we obtain the event inclusion (as at least one of the two terms must be greater than $\gamma/2$ for their sum to be greater than $\gamma$):
\[
\begin{aligned}
\{|d_S-d_C|\ge \gamma\}
\subseteq\;&\{|p_{1,S}-p_{1,C}|\ge \gamma/2\}\\
&\cup\{|p_{0,S}-p_{0,C}|\ge \gamma/2\}.
\end{aligned}
\]
Applying the union bound and the concentration inequality above gives
\begin{align*}
\Pr(|d_S-d_C|\ge \gamma \mid n_0,n_1)
&\le
\Pr\!\left(|p_{1,S}-p_{1,C}|\ge \gamma/2 \mid n_1\right)\\
&\quad+
\Pr\!\left(|p_{0,S}-p_{0,C}|\ge \gamma/2 \mid n_0\right)\\
&\le
2\exp\!\left(-2n_1(\gamma/2)^2\right)\\
&\quad+
2\exp\!\left(-2n_0(\gamma/2)^2\right)\\
&\le
4\exp\!\left(-\frac{\gamma^2}{2}\,n_{\min}\right).
\end{align*}

Set $\gamma := \sqrt{\frac{2\ln(4/\delta)}{n_{\min}}}$. Then
\[
\begin{aligned}
\Pr(|d_S-d_C|\ge \gamma \mid n_0,n_1)
&\le 4\exp(-\ln(4/\delta)) \\
&= \delta.
\end{aligned}
\]
Now assume $|d_C|\le \mtol-\gamma$.
Whenever $|d_S-d_C|<\gamma$, the triangle inequality yields
\[
|d_S|\le |d_C|+|d_S-d_C| < (\mtol-\gamma)+\gamma=\mtol.
\]
Hence $\{|d_S|>\mtol\}\subseteq \{|d_S-d_C|\ge \gamma\}$ and therefore
\[
\Pr(|d_S|>\mtol \mid n_0,n_1)\le \delta,
\]
which is equivalent to $\Pr(|d_S|\le \mtol \mid n_0,n_1)\ge 1-\delta$.
\end{proof}

We note that, since sampling occurs without replacement from the finite candidate set $C$, these bounds can actually be tightened through a finite-population correction. 
For completeness, we detail this refinement in Appendix~\ref{app:finite-population}, while retaining the ordinary Hoeffding bound in the main text for simplicity and as a conservative worst-case bound.
\section{Illustrative Extension with Equal Opportunity}
\label{app:equal-opportunity}

The analysis of Theorem~\ref{thm:n_manipulation_respir_audit} relies on two properties of demographic parity: the audited metric is a difference between two empirical rates and the effect of a single output flip on this difference can be bounded. 
The same argument applies to other group fairness metrics after identifying the appropriate conditional populations and single-flip sensitivity. 
We illustrate this adaptation for the equal opportunity metric \cite{hardt2016equality}, which quantifies differences in true positive rates across protected groups.

Assume that each $x\in C$ is associated with a binary ground-truth
label $Y(x)\in\{0,1\}$. 
For each protected group $a\in\{0,1\}$,
define
\[
C_a^+
:=
\{x\in C:A(x)=a,\ Y(x)=1\},
\qquad
N_a^+
:=
|C_a^+|,
\]
and similarly
\[
S_a^+
:=
\{x\in S:A(x)=a,\ Y(x)=1\},
\qquad
n_a^+
:=
|S_a^+|.
\]
The true positive rate of group $a$ on $C$ is
\[
p_{a,C}^{\mathrm{EO}}
:=
\frac{1}{N_a^+}
\sum_{x\in C_a^+}h(x),
\]
and the equal opportunity gap on $C$ is
\[
d_C^{\mathrm{EO}}
:=
p_{1,C}^{\mathrm{EO}}
-
p_{0,C}^{\mathrm{EO}}.
\]
The quantities $p_{a,S}^{\mathrm{EO}}$ and
$d_S^{\mathrm{EO}}$ are defined analogously on $S$. As in Section \ref{subsec:black-box_audit}, we omit the classifier $h$ in notation since it is clear from context. Finally, let
\[
N_{\min}^+
:=
\min(N_0^+,N_1^+),
\qquad
n_{\min}^+
:=
\min(n_0^+,n_1^+).
\]

\begin{theorem}[\ac{respir} manipulation bound for equal opportunity]
\label{thm:respir-equal-opportunity}
Let $C\subset\mathcal{X}$ be the candidate set and let
$S\subset C$ be sampled uniformly at random without replacement.
Assume that $N_a^+,n_a^+>0$ for each $a\in\{0,1\}$. Let
$d_{C,\text{true}}^{\mathrm{EO}}$ denote the equal opportunity gap
of the provider's model on $C$ before manipulation, and let
$\varepsilon$ be the audit threshold.

For $\delta\in(0,1/2]$, define
\[
\gamma_{\mathrm{EO}}(\delta)
:=
\sqrt{
\frac{2\ln(4/\delta)}{n_{\min}^+}
}.
\]
If the post-manipulation candidate-set gap satisfies
\[
|d_C^{\mathrm{EO}}|
\leq
\varepsilon-\gamma_{\mathrm{EO}}(\delta),
\]
then
\[
\Pr\!\left(
    |d_S^{\mathrm{EO}}|\leq\varepsilon
    \,\middle|\,
    n_0^+,n_1^+
\right)
\geq 1-\delta.
\]

Consequently, assuming that the original (non-manipulated) model would fail the
equal opportunity audit, \emph{i.e.},
\[
|d_{C,\text{true}}^{\mathrm{EO}}|>\varepsilon,
\]
any set of output flips that enforces this sufficient condition must
contain at least
\[
m_{\mathrm{EO}}
\geq
\left\lceil
\left(
|d_{C,\text{true}}^{\mathrm{EO}}|
-\varepsilon
+\gamma_{\mathrm{EO}}(\delta)
\right)N_{\min}^+
\right\rceil
\]
flipped predictions on examples with $Y(x)=1$.
\end{theorem}

\begin{proof}
Fix a realization of $(n_0^+,n_1^+)$. Conditional on $n_a^+$,
the set $S_a^+$ is a uniformly random subset of size $n_a^+$ of
$C_a^+$. 
Thus, $p_{a,S}^{\mathrm{EO}}$ is the average of
$n_a^+$ binary predictions sampled uniformly without replacement
from the $N_a^+$ predictions in $C_a^+$.

Hoeffding's inequality for sampling without replacement gives, for
every $t>0$,
\[
\Pr\!\left(
    \left|
    p_{a,S}^{\mathrm{EO}}
    -
    p_{a,C}^{\mathrm{EO}}
    \right|
    \geq t
    \,\middle|\,
    n_a^+
\right)
\leq
2\exp(-2n_a^+t^2).
\]
Using
\[
d_S^{\mathrm{EO}}-d_C^{\mathrm{EO}}
=
\left(
p_{1,S}^{\mathrm{EO}}-p_{1,C}^{\mathrm{EO}}
\right)
-
\left(
p_{0,S}^{\mathrm{EO}}-p_{0,C}^{\mathrm{EO}}
\right),
\]
the triangle inequality and a union bound yield
\[
\Pr\!\left(
    \left|
    d_S^{\mathrm{EO}}-d_C^{\mathrm{EO}}
    \right|
    \geq\gamma
    \,\middle|\,
    n_0^+,n_1^+
\right)
\leq
4\exp\left(
    -\frac{\gamma^2}{2}n_{\min}^+
\right).
\]
Setting
\[
\gamma
=
\gamma_{\mathrm{EO}}(\delta)
=
\sqrt{
\frac{2\ln(4/\delta)}{n_{\min}^+}
}
\]
makes the right-hand side equal to $\delta$. 
Hence, if
\[
|d_C^{\mathrm{EO}}|
\leq
\varepsilon-\gamma_{\mathrm{EO}}(\delta),
\]
then $|d_S^{\mathrm{EO}}|\leq\varepsilon$ with conditional
probability at least $1-\delta$.

We now bound the number of output flips required to enforce this
condition. 
Assume without loss of generality that
$d_{C,\text{true}}^{\mathrm{EO}}>0$. 
Flipping a prediction for an
example in $C_1^+$ from $1$ to $0$ decreases the equal opportunity
gap by $1/N_1^+$. 
Similarly, flipping a prediction for an example in
$C_0^+$ from $0$ to $1$ decreases the gap by $1/N_0^+$. 
In contrast, flips on examples with $Y(x)=0$ do not affect equal opportunity. 
Therefore,
a single output flip can decrease the gap by at most
\[
\max\left(
    \frac{1}{N_0^+},
    \frac{1}{N_1^+}
\right)
=
\frac{1}{N_{\min}^+}.
\]
Reducing the initial gap to at most
$\varepsilon-\gamma_{\mathrm{EO}}(\delta)$ thus requires at least
\[
\left\lceil
\left(
|d_{C,\text{true}}^{\mathrm{EO}}|
-\varepsilon
+\gamma_{\mathrm{EO}}(\delta)
\right)N_{\min}^+
\right\rceil
\]
output flips, which proves the result.
\end{proof}

This example illustrates the general adaptation of Theorem~1. 
For a fairness metric expressed as a difference between conditional rates, the relevant group sizes are those of the conditioning strata while the manipulation bound is determined by the maximum change in the metric caused by a single output flip. 
Metrics involving several simultaneous rate constraints, such as equalized odds, can be handled similarly by applying concentration bounds to each rate and combining them through an appropriate union bound.

\section{Refining the Bound through a Finite-Population Correction}
\label{app:finite-population}

Theorems~\ref{thm:n_manipulation_respir_audit} and \ref{cor:manip_sufficient_condition} apply a Hoeffding bound that is valid under sampling without replacement but does not exploit the finite population sizes $N_0$ and $N_1$. 
As a complementary result, we apply the finite-population refinement of the Hoeffding--Serfling inequality proposed by \citet[Theorem~2.4 and Corollary~2.5]{bardenet2015concentration}, which builds on and slightly improves the original bound of \citet{serfling1974}.

For each $a\in\{0,1\}$, define
\[
\rho_a :=
\begin{cases}
1-\dfrac{n_a-1}{N_a},
& \text{if } n_a\leq N_a/2, \\[6pt]
\left(1-\dfrac{n_a}{N_a}\right)
\left(1+\dfrac{1}{n_a}\right),
& \text{if } n_a>N_a/2,
\end{cases}
\]
and
\[
\gamma_{\mathrm{FPC}}(\delta)
:=
\sqrt{\frac{\rho_1\ln(4/\delta)}{2n_1}}
+
\sqrt{\frac{\rho_0\ln(4/\delta)}{2n_0}}.
\]

\begin{theorem}[Finite-population refinement of Theorem~1]
\label{thm:finite-population}
Consider the setting of Theorem~1 and fix $\delta\in(0,1)$. If the post-manipulation candidate-set gap satisfies
\[
|d_C|
\leq
\varepsilon-\gamma_{\mathrm{FPC}}(\delta),
\]
then
\[
\Pr\!\left(
|d_S|\leq\varepsilon
\,\middle|\,
n_0,n_1
\right)
\geq 1-\delta.
\]
Moreover, enforcing this sufficient condition from an initial gap $d_{C,\text{true}}$ requires at least
\[
m_{\mathrm{FPC}}
\geq
\left\lceil
\left(
|d_{C,\text{true}}|
-\varepsilon
+\gamma_{\mathrm{FPC}}(\delta)
\right)N_{\min}
\right\rceil
\]
output flips, where $N_{\min}:=\min(N_0,N_1)$.
\end{theorem}

\begin{proof}
Conditioning on $n_a$, apply \citet[Corollary~2.5]{bardenet2015concentration} to the finite population of $N_a$ binary predictions in group $a$. 
In their notation, the population mean is $p_{a,C}$, the sample mean is $p_{a,S}$, the sample size is $n_a$, and the range satisfies $b-a=1$. 
Thus, with probability at least $1-\delta$,
\[
p_{a,S}-p_{a,C}
\leq
\sqrt{\frac{\rho_a\log(1/\delta)}{2n_a}}.
\]
Equivalently, setting
$\delta=\exp(-2n_at^2/\rho_a)$ gives, for every $t>0$,
\[
\Pr\!\left(
p_{a,S}-p_{a,C}>t
\,\middle|\,
n_a
\right)
\leq
\exp\!\left(-\frac{2n_at^2}{\rho_a}\right).
\]
The same bound applies to the opposite deviation
$p_{a,C}-p_{a,S}$. 
A union bound over the two tails therefore yields
\begin{align}
\Pr\!\left(
|p_{a,S}-p_{a,C}|\geq t
\,\middle|\,
n_a
\right)
\leq
2\exp\!\left(
-\frac{2n_at^2}{\rho_a}
\right).\label{eq:tighter_bound_1}
\end{align}

For each $a\in\{0,1\}$, set
\[
t_a
:=
\sqrt{\frac{\rho_a\ln(4/\delta)}{2n_a}}.
\]
Substituting $t=t_a$ into Equation \eqref{eq:tighter_bound_1} gives
\[
\Pr\!\left(
|p_{a,S}-p_{a,C}|>t_a
\,\middle|\,
n_a
\right)
\leq
2\exp\!\left(-\ln(4/\delta)\right)
=
\frac{\delta}{2}.
\]
A union bound over the two protected groups therefore shows that,
with probability at least $1-\delta$,
\[
|p_{1,S}-p_{1,C}|\leq t_1
\qquad\text{and}\qquad
|p_{0,S}-p_{0,C}|\leq t_0.
\]
On this event,
\begin{align*}
|d_S-d_C|
&=
\left|
(p_{1,S}-p_{1,C})
-
(p_{0,S}-p_{0,C})
\right| \\
&\leq
|p_{1,S}-p_{1,C}|
+
|p_{0,S}-p_{0,C}| \\
&\leq
t_1+t_0
=
\gamma_{\mathrm{FPC}}(\delta).
\end{align*}
Consequently, if
\[
|d_C|
\leq
\varepsilon-\gamma_{\mathrm{FPC}}(\delta),
\]
then
\[
|d_S|
\leq
|d_C|+|d_S-d_C|
\leq
\varepsilon
\]
with probability at least $1-\delta$.

Finally, a single output flip changes $d_C$ by at most
\[
\max\left(\frac{1}{N_0},\frac{1}{N_1}\right)
=
\frac{1}{N_{\min}}.
\]
Enforcing the sufficient condition above therefore requires reducing
the candidate-set gap by at least
\[
|d_{C,\text{true}}|
-\varepsilon
+\gamma_{\mathrm{FPC}}(\delta).
\]
It consequently requires at least
\[
\left\lceil
\left(
|d_{C,\text{true}}|
-\varepsilon
+\gamma_{\mathrm{FPC}}(\delta)
\right)N_{\min}
\right\rceil
\]
output flips as claimed.
\end{proof}

Since $\rho_a\leq1$ and $n_a\geq n_{\min}$ for each
$a\in\{0,1\}$,
\begin{align*}
\gamma_{\mathrm{FPC}}(\delta)
&=
\sqrt{\frac{\rho_1\ln(4/\delta)}{2n_1}}
+
\sqrt{\frac{\rho_0\ln(4/\delta)}{2n_0}} \\
&\leq
\sqrt{\frac{\ln(4/\delta)}{2n_1}}
+
\sqrt{\frac{\ln(4/\delta)}{2n_0}} \\
&\leq
\sqrt{\frac{2\ln(4/\delta)}{n_{\min}}}.
\end{align*}
The finite-population correction thus yields a tighter concentration margin than the ordinary Hoeffding bound, particularly when a large fraction of either protected group is sampled. 
This results in a less conservative sufficient threshold for the number $m_{\mathrm{FPC}}$ of output flips required for the platform to pass the audit with sufficiently high probability.

\end{document}